\documentclass{article}

\usepackage[utf8]{inputenc}

\usepackage{amsmath}
\usepackage{amsfonts}
\usepackage{amssymb}
\usepackage{bbm, bm}
\usepackage{amsthm}
\usepackage{color}
\usepackage{url}
\usepackage{graphicx,color,wrapfig}
\usepackage[ruled,vlined]{algorithm2e}

\usepackage{subfigure}

\newtheorem{theorem}{Theorem}[section]

\newtheorem{assumption}{Assumption}[section]

\title{Decentralized Bayesian Learning with Metropolis-Adjusted Hamiltonian Monte Carlo}
\author{Vyacheslav Kungurtsev\thanks{First and second authors contributed equally to the paper.} \\ Department of Computer Science \\ Czech Technical University in Prague \\ \texttt{vyacheslav.kungurtsev@fel.cvut.cz} \\ 
Adam Cobb$^*$ \\ SRI International \\ \texttt{adam.cobb@sri.com} \\ Tara Javidi \\ Electrical and Computer Engineering \\ University of California, San Diego \\ \texttt{tjavidi@ucsd.edu} \\ Brian Jalaian \\ DEVCOM Army Research Laboratory \\ \texttt{brian.a.jalaian.civ@mail.mil} }
\date{November 2020}

\begin{document}
\maketitle
\begin{abstract}
Federated learning performed by a decentralized networks of agents is becoming increasingly important with the prevalence of embedded software on autonomous devices. Bayesian approaches to learning benefit from offering more information as to the uncertainty of a random quantity, and Langevin and Hamiltonian methods are effective at realizing sampling from an uncertain distribution with large parameter dimensions. Such methods have only recently appeared in the decentralized setting, and either exclusively use stochastic gradient Langevin and Hamiltonian Monte Carlo approaches that require a diminishing stepsize to asymptotically sample from the posterior and are known in practice to characterize uncertainty less faithfully than constant step-size methods with a Metropolis adjustment, or assume strong convexity properties of the potential function. We present the first approach to incorporating constant stepsize Metropolis-adjusted HMC in the decentralized  sampling framework, show theoretical guarantees for consensus and probability distance to the posterior stationary distribution, and demonstrate their effectiveness numerically on standard real world problems, including decentralized learning of neural networks which is known to be highly non-convex.  
\end{abstract}

\section{Introduction}\label{s:intro}

 Bayesian learning and inference 
have received a lot of attention in the literature as a principled way to avoid over-fitting and/or to quantify uncertainties in an integrated fashion
by estimating the posterior distribution of the model parameters rather than a point-wise estimate. However, analytical solutions of exact posterior or
sampling from the exact posterior is often impossible due to the intractability of the evidence. Therefore, Bayesian inference often relies on sampling methods broadly known as Markov Chain Monte Carlo (MCMC). To
this effect, consider the following computation of a Bayesian posterior across a decentralized data set:

\begin{equation}\label{eq:prob}
\pi(\omega)= p(\omega|\{X,Y\}) = \prod\limits_{i=1}^m \frac{p(\{X_i,Y_i\}|\omega)}{p(X_i,Y_i)}\pi_0 (\omega), 
\end{equation}
where $i\in[m]$ indicates the index from a set of partitioned data, distributed across a set of agents connected on a communication graph. Each of the agents has access to data samples $\{X_i,Y_i\}$ that are in general expected to have different distributions, and it is expected that during the process of finding the posterior, the labels are distributed in a heterogeneous way across agents, however, we wish for them all to be able to make inference on all types of samples. 

When the likelihood function $p(\{X_i,Y_i\}|\omega)$ arises from M-estimation for an exponential family and the prior is conjugate, the resulting distribution of $\pi(\omega)$ is known to be of Gibbs form, 
$\pi(\omega) \propto {e^{-U(\omega)}}$, where $U$ is referred to as a potential function for $\omega$. 
This makes the process amenable to techniques arising in stochastic differential equations (SDEs), namely Langevin dynamics and Hamiltonian Monte Carlo, which implement a discretization of certain SDEs. The stationary distribution of these processes with a gradient term $\nabla U(\omega)$ as the drift term in the SDE is precisely the Gibbs distribution above. For high dimensional parameter vectors, these approaches have been shown theoretically and numerically to scale better. Ergodicity of these methods was established in~\cite{roberts1996exponential} and~\cite{durmus2017convergence}. For these class of methods, asymptotic convergence to the stationary distribution, regardless of its form, has been established, as well as geometric convergence in measure for when the target distribution $\pi(\omega) \propto {e^{-U(\omega)}}$ is strongly log-concave, i.e.\ $U$ is strongly convex and smooth. However, a Metropolis acceptance-rejection step is necessary to correct for the discretization bias. 

Recently, the so-called ``stochastic gradient" Langevin Monte Carlo (SGLD) and Hamiltonian Monte Carlo (SGHMC) approaches have become popular with the seminal work of~\cite{welling2011bayesian}. This was developed for ``big-data" contexts wherein it is preferable to subsample from the dataset at each iteration, making these methods resemble stochastic gradient descent. By incorporating a diminishing stepsize akin to stochastic approximation methods in the implementation of the discretization, the Metropolis adjustment becomes no longer necessary. Theoretically, although there is rich theory in the convergence properties of SGLD and SGHMC methods based on techniques of optimization in probability measure spaces (e.g.,~\cite{teh2016consistency, zhang2017hitting, chen2020stationary} and many others), guarantees exist primarily for log concave distributions (and in the most general case, a log Sobolev or Poincar{\' e} inequality), eschewing their applicability to, e.g., Bayesian Neural Networks. Empirically, it has been observed that HMC in general performs better than SGHMC in terms of characterizing the uncertainty, which can be seen visually as mapping the posterior more faithfully, and by considering performance metrics of statistical quality (log likelihood) over optimization (RMSE)~\cite{izmailov2021bayesian, cobb2020scaling}. An argument based on fundamentals of HMC as to why subsampling is expected to degrade the performance is given in~\cite{betancourt2015fundamental}. 

Thus the constant stepsize Metropolis adjusted Langevin and HMC methods, as compared to their diminishing stepsize stochastic gradient variants stand as complementary techniques with distinct comparative advantages and disadvantages. Namely, for handling the computational load of large datasets and obtaining accurate point estimates, the stochastic gradient variants are to be preferred, however, when statistical accuracy of the entire posterior is prioritized or the potential exhibits nonconvexities, the constant stepsize Metropolis adjusted variants are the more appropriate choice.

In this paper, we consider 
sampling in a decentralized context motivated by extensions of federated learning where the decentralized Bayesian extension~\cite{lalitha2019decentralized}--with a variational inference approach--are shown to provide further robustness. Two main challenges exhibited are 1) handling larger models and 2) forgetting previously learned models in the presence of heterogeneous data. Both are standard concerns~\cite{shoham2019overcoming}. By using a schema incorporating methods amenable to high dimension, we hope to address these difficulties. As such, our work is contextualized within a handful of studies focusing on SDE methods for decentralized sampling. In particular,~\cite{kungurtsev2020stochastic} provided convergence and consensus guarantees for decentralized Langevin dynamics when the posterior is strongly convex.
A diminishing stepsize is needed to ensure asymptotic consensus, placing this work conceptually within the stochastic gradient variants. 
Subsequent work extended this to to stochastic gradient Hamiltonian Dynamics in~\cite{gurbuzbalaban2020decentralized}, again for when $U$ is strongly convex. 
Extension for non-log-concave potentials, but satisfying a Sobolev inequality, was given in~\cite{parayil2020decentralized}. More recently, directed graphs are considered~\cite{kolesov2021decentralized}. The theoretical guarantees in these works take the optimization over measures approach to the analysis, and assumptions on the growth of the potentials (either strong convexity of the Sobolev inequality) are required, and in most cases a diminishing stepsize is used. Since a Metropolis adjustment step is not present in any of these methods, this is natural to be expected. 

In this paper, we complete the program of extending SDE based sampling methods to the decentralized context by presenting the first algorithm, together with guarantees, for implementing sampling from the constant stepsize Metropolis-adjusted framework on a decentralized platform. This presents two significant methodological challenges: achieving consensus without a diminishing stepsize and performing a Metropolis acceptance-rejection procedure in a decentralized way. By carefully incorporating techniques inspired from the decentralized optimization literature along with developing a novel method for approximating the Metropolis acceptance ratio, we present an algorithm that performs HMC for decentralized data. Conceptually we seek to achieve that each agent, asymptotically, performs HMC, however, with gradient evaluations on the entire data set through gradually diffusing information. Theoretically we prove a bound in expected $L2$ error between our algorithm and the chain as generated by the classical HMC kernel, quantitatively bounding the discrepancy in terms of the step-size and the total number of sampling iterations. Thus, modulo some adjustable level of error, our algorithm approximates the probabilistic convergence properties of HMC, as for instance derived in~\cite{durmus2017convergence}, which showed ergodicity for general potentials and geometric ergodicity for strongly convex potentials.

The salient technical features of our proposed solution, specifically, is of three folds. First, we integrate the concept of \emph{(gradient) tracking} as appearing in decentralized optimization~\cite{di2016next, pu2020distributed} for sampling, for the first time, and indicate its analogous benefit of allowing each agent to asymptotically perform, effectively, optimization on the same global function. This is in contrast to  a repeated ``zig-zag" pattern of local optimization and consensus pushing in disparate directions at each iterate, even more crucial for sampling since effectively each iterate is important, rather than some specific limit points.
Secondly, we obtain guarantees of asymptotic consensus by incorporating multiple mixing rounds in each iteration~\cite{berahas2018balancing}. Finally, we utilize a Metropolis adjustment step to correct for the discretization bias encountered in 
constant-step-size HMC. In order to compute an approximation to the full posterior in a decentralized context, we introduce a technique of tracking a decentralized second order Taylor approximation of the posterior, using only Hessian-vector products, and similarly \emph{tracking} to communicate and aggregate local information across the network.

\subsection{Background and Preliminaries}

The communication network of the agents is modeled as a fixed undirected graph $\mathcal{G}\triangleq (\mathcal{V},\mathcal{E})$ with 
vertices $\mathcal{V}\triangleq \{1,..,m\}$ and $\mathcal{E}\triangleq\{(i,j)|i,j\in\mathcal{V}\}$ representing the agents and
communication links, respectively. We assume that the graph $\mathcal{G}$ is strongly connected. We note by $\mathcal{N}_i$ the
neighbors of $i$, i.e., $\mathcal{N}_i = \{j:(i,j)\in\mathcal{E}\}$. We define the graph Laplacian matrix $\mathbf{L}=\mathbf{I}-\mathbf{W}$, where $\mathbf{W}=\mathbf{A}\otimes \mathbf{I}$ with
$\mathbf{A}$ satisfying $\mathbf{A}_{ij}\neq 0$ if $(i,j)\in\mathcal{E}$ and $\mathbf{A}_{ij}=0$ otherwise. We assume that $\mathbf{W}$ is double stochastic (and symmetric, since the graph is undirected). The eigenvalues of $\mathbf{L}$ are real and can be sorted in a 
nonincreasing order $1=\lambda_1(\mathbf{L})> \lambda_2(\mathbf{L})\ge ...\ge \lambda_n(\mathbf{L})\ge 0$. Defining, $\beta:= \lambda_2(\mathbf{L})$
we shall make the following assumption,
\begin{assumption}\label{as:beta}
It holds that,
\[
\beta<1
\]
\end{assumption}
We shall define $\bar{\beta}$ to be the smallest eigenvalue of $\mathbf{L}$ that is nonzero.

\section{Decentralized Metropolis-Adjusted Hamiltonian Monte Carlo}\label{s:alg}

We are now ready to introduce our approach, which we present formally in Algorithm~\ref{alg:mainalg}. As part of our approach we introduce two key concepts: \emph{gradient tracking}, and a \emph{Taylor approximation}. We will describe both components in the subsequent sections before describing how they combine in Algorithm~\ref{alg:mainalg}.

\subsection{Gradient Tracking}

Gradient tracking is a technique that enables every agent to obtain an estimate of the full gradient of the potential in a decentralized setting. While originally introduced and implemented for decentralized nonconvex optimization \cite{di2016next}, here we use gradient tracking to enable gradient-based sampling of a potential, whereby each agent only has access to its own portion of the data. In our method, each agent tracks both first and second order information of the global potential, which we denote as $\mathbf{g}^t_i$ and $\aleph^t_i$ respectively.

We index the individual agents by $i$ and the iteration counter by $t$. To perform gradient tracking, a weighted average of communicating neighbors is balanced at each iteration to accumulate the gradient. This weighted average is determined by the mixing matrix, $\bm{W}$, of the network. As a result, each agent computes its local gradient of the potential, $\nabla p(\bm{\omega}^t|X_i,Y_i)$, and updates its global estimation of the gradient using its current local gradient, its previous local gradient, and its current estimation of the global gradient. We display this update equation as follows:
\begin{equation}\label{eq:1track}
    \mathbf{g}_i^{t+1} = \mathbf{W} \left(\mathbf g_i^t+\nabla p(\bm{\omega}^t_i|X_i,Y_i)-\nabla p(\bm{\omega}^{(t-1)}_i|X_i,Y_i)\right).
\end{equation}

We intend that as the parameter estimates reach consensus (agreement across agents) we expect that $\bm{g}$ to also reach consensus and in effect, every agent samples HMC on the entire dataset.
Since in the context of sampling, each iteration is counted rather than just asymptotic limit points, we believe this should improve posterior characterization and avoid the characteristic zig-zagging of decentralized algorithms without gradient tracking.




\subsection{Taylor Approximated Metropolis Adjustment}

The second key component of our approach builds on gradient tracking by utilizing the tracked second order term, $\aleph^t_i$. One of the challenges in performing decentralized MCMC, is the requirement of a decentralized Metropolis-Hastings (MH) step. Since this involves evaluation of the posterior over the whole data set, it is unclear how this can be done with each agent only having access to its local data. Understandably, all previous approaches to decentralized sampling either use diminishing step-sizes~\cite{kungurtsev2020stochastic} or accept the level of bias according to the discretization error~\cite{gurbuzbalaban2020decentralized}.

In order to perform a decentralized MH step, we first introduce the Metropolis adjustment step, which takes the difference between the Hamiltonian at the current time step, $H(\bm{\omega}^t,\bm{p}^t)$ and the proposed Hamiltonian with new parameters, $H(\bm{\omega}^*,\bm{p}^*)$, and accepts the new parameters according to the acceptance ratio $\rho$, where
$\log \rho = \min \{0, - H(\bm{\omega}^*,\bm{p}^*) + H(\bm{\omega},\bm{p})\}$. In the decentralized setting each agent does not have access to the full Hamiltonian and therefore to overcome this problem, we can approximate the acceptance ratio using the tracked first and second order terms. As a result we perform a Taylor expansion to approximate our acceptance step. 

To derive the approximated acceptance ratio, we first take a Taylor expansion of the Hamiltonian at the proposed time step $\{\bm{\omega}^*,\bm{p}^*\}$ and define both $\Delta \bm{\omega} = \bm{\omega}^* - \bm{\omega}^t$ and $\Delta \mathbf{p} = \mathbf{p}^* - \mathbf{p}^t$. Then,
$$H(\bm{\omega}^*,\bm{p}^*) = H(\bm{\omega}^t, \bm{p}^t) + \Delta \bm{\omega} \frac{\partial H }{\partial \bm{\omega}} + \Delta \mathbf{p} \frac{\partial H }{\partial \mathbf{p}} + \frac 12 \Delta \bm{\omega} ^T \frac{\partial^2 H}{\partial \bm{\omega}^2} \Delta \bm{\omega} + \frac 12 \Delta \mathbf{p}^T \frac{\partial^2 H}{\partial \bm{p}^2} \Delta \mathbf{p},$$
where the first and second order derivatives are evaluated at $\{\bm{\omega}^t,\bm{p}^t\}$. As a result we can write the acceptance ratio as:
\begin{align}
   \log \rho = & \min \{0, - H(\bm{\omega}^*,\mathbf{p}^*) + H(\bm{\omega},\mathbf{p})\}\notag\\
    \log \rho \approx & \min \Bigg \{ 0, - \Delta \bm{\omega} \frac{\partial H }{\partial \bm{\omega}} - \Delta \mathbf{p} \frac{\partial H }{\partial \mathbf{p}} - \frac 12 \Delta \bm{\omega} ^T \frac{\partial^2 H}{\partial \bm{\omega}^2} \Delta \bm{\omega} - \frac 12 \Delta \mathbf{p}^T \frac{\partial^2 H}{\partial \bm{p}^2} \Delta \mathbf{p} \Bigg \}. \notag
\end{align}
We are performing numerical integration with a single first order Euler update step, which gives $\mathbf{p}^* = \mathbf{p}^t + \epsilon \mathbf{g}^t$, and  $\bm{\omega}^* = \bm{\omega}^t + \epsilon \mathbf{p}^*$. Therefore, $\Delta \bm{\omega} = \epsilon \mathbf{p}^*$ and  $\Delta \mathbf{p}^* = \epsilon \mathbf{g}^t$. As a result the first order terms cancel each other out  leaving the quadratic terms. The second order derivative with respect to the momentum is a vector of ones (assuming the mass matrix is the identity) and the second order derivative with respect to the parameters is the Hessian of the unnormalized potential, $\frac{\partial^2 U}{\partial \bm{\omega}^2}$. Rather than directly computing the full Hessian and using tracking in the same manner as Equation~\eqref{eq:1track}, we can track the quadratic term directly and take advantage of the speed-up associated with the vector-Hessian product. As a result, each agent tracks this second order term using the equation
$$
\aleph_i^{t+1} = \mathbf{W} \left(\mathbf \aleph_i^t + (\mathbf{p}_i^t)^T \frac{\partial^2 U_i}{\partial \bm{\omega}^2}(\bm{\omega}^t_i) \mathbf{p}_i^t-(\mathbf{p}_i^{t-1})^T \frac{\partial^2 U_i}{\partial \bm{\omega}^2}(\bm{\omega}^{t-1}_i) \mathbf{p}_i^{t-1}\right)
$$
which will converge asymptotically to the global value. Thus we can set our our Metropolis acceptance to
\[
\log \rho \approx \min \{ 0 , -\aleph_i-\epsilon^2 \|g_i\|^2\}
\]

\subsection{The Algorithm}

Algorithm \ref{alg:mainalg} describes the decentralized Metropolis-adjusted Langevin algorithm, which combines the gradient tracking of first and second order terms as well as the Taylor approximated Metropolis step. Each agent starts by sampling its own momentum and then computes its own local gradient and quadratic second order term. Then each agent updates its estimated global gradient and quadratic terms using gradient tracking. Finally, each agent then makes a single step in the augmented parameter space (including both the model parameters and the momentum) and accepts with probability according to the approximated Metropolis acceptance ratio. We note one additional feature that we introduce, 
namely, for both consensus and tracking computations, we perform $t$-quantity mixing steps, i.e., the number of mixing steps grows geometrically with the iterations. This technique was introduced in~\cite{berahas2018balancing} to ensure asymptotic consensus without annealing the step-size. Although we shall make use of this in the convergence analysis, in our implementation, we increase the mixing to computation ratio more gradually, with, e.g., a thousand iterations before an increase in the ratio from one-to-one to two-to-one, etc..

\begin{algorithm}[h!]
\SetAlgoLined
\KwResult{Samples from the posterior distribution $p(\omega\vert X, Y)$.}
 initialization: $\mathbf{W}, \{\mathcal{L}_i, \bm{\omega}_i^0, \mathbf{g}^0_i\}_{i=1}^N,  \epsilon, T $\;
 \For{$t$ in $1,\dots,T$}{
  \For{$i$ in $1,\dots,m$}{
  Sample agent-wise momentum\;
  $\mathbf{p}_i^t = \mathcal{N}(\mathbf{0}, \mathbf{M})$\;
  Compute local gradient\;
  $\mathbf{g}_{\text{local},i}^t = \nabla p(\bm{\omega}^{k}_i|\mathbf{X}_i,\mathbf{Y}_i)$\;
  Compute local Quadratic term\;
  $\aleph_{\text{local},i}^t = (\mathbf{p}_i^t)^T \frac{\partial^2 U_i}{\partial \bm{\omega}^2} \mathbf{p}_i^t \left(= - (\mathbf{p}_i^t)^T \frac{\partial^2 \mathcal{L}_i}{\partial \bm{\omega}^2} \mathbf{p}_i^t \right) $;  
  }
  \eIf{Gradient Tracking}{
   $\mathbf{g}_{i}^{t+1} = \sum_{j\in\mathcal{N}_i} \left(\mathbf{W}^t\right)_{ij}\left(\left[\mathbf{g}_{i}^{t}\right]+\mathbf{g}_{\text{local},i}^t-\mathbf{g}_{\text{local},i}^{t-1}\right)$\;
   $\aleph_{i}^{t+1} = \sum_{j\in\mathcal{N}_i} \left(\mathbf{W}^t\right)_{ij}\left(\left[\aleph_{i}^{t}\right]+\aleph_{\text{local},i}^t-\aleph_{\text{local},i}^{t-1}\right)$\;
   }{
   $\mathbf{g}_{i}^{t+1} = \mathbf{g}_{\text{local},i}^t$\;
   $\aleph_{i}^{t+1} = \aleph_{\text{local},i}^t$\;
  }
 \For{$i$ in $1,\dots,m$}{
    $\bm{\omega}_i^*, \mathbf{p}_i^* = \text{EulerUpdate}(\bm{\omega}_i^t, \mathbf{p}_i^t, \epsilon, \mathbf{g}_{i}^{t+1})$\;
    Metropolis Hastings Step\;
    $\Delta H = - 0.5 \epsilon^2 \left( \aleph_{i}^{t+1} + (\mathbf{g}_{i}^{t+1})^T \mathbf{g}_{i}^{t+1}\right)$\;
    $\rho = \min(0,  \Delta H)$\;
    $u \sim \mathcal{U}(0,1)$\;
    \eIf{$\rho \geq \log u $}{
    $\bm{\omega}_i^{t+1}, \mathbf{p}_i^{t+1} = \bm{\omega}_i^*, \mathbf{p}_i^*$\;
    }{
    $\bm{\omega}_i^{t+1}, \mathbf{p}_i^{t+1} = \bm{\omega}_i^t, \mathbf{p}_i^t$\;
    }
    
 }
 \If{Consensus Step}{
    $\bm{\omega}_i^{t+1} = \sum_{j\in\mathcal{N}_i} \left(\mathbf{W}^t\right)_{ij} \bm{\omega}_i^{t+1}$
 }
 
 }
 
 \caption{Decentralized Metropolis-adjusted Langevin algorithm (D-MALA)}\label{alg:mainalg}
\end{algorithm}

\section{Convergence Analysis}\label{s:conv}
To analyze the convergence of Algorithm~\ref{alg:mainalg}, 
we first write the vectorized expression for the iterate sequence generated by the Algorithm as follows. We shall denote $\bm{\omega}^t = \begin{pmatrix} \bm{\omega}_1^T & ... & \bm{\omega}^T_m \end{pmatrix}^T$ as the set of parameter vectors $\{\bm{\omega}_i\}$ stacked together, and similarly for $\bm{p}^t$, $\aleph^t$ and $\bm{g}^t$.  
\begin{equation}\label{eq:hmcvecstackctaalg}
\begin{array}{l}
G(\bm{\omega}^t) = \begin{pmatrix} \nabla p(\bm{\omega}^t_1|X_1,Y_1) & \nabla p(\bm{\omega}^t_2|X_2,Y_2) & ... & \nabla p(\bm{\omega}^t_m|X_m,Y_m) \end{pmatrix}^T \\
H(\bm{\omega}^t) = \begin{pmatrix} (\mathbf{p}_1^t)^T \frac{\partial^2 U_1}{\partial \bm{\omega}^2}(\bm{\omega}^t_1) \mathbf{p}_1^t & (\mathbf{p}_2^t)^T \frac{\partial^2 U_2}{\partial \bm{\omega}^2}(\bm{\omega}^t_2) \mathbf{p}_2^t & ... & (\mathbf{p}_m^t)^T \frac{\partial^2 U_m}{\partial \bm{\omega}^2}(\bm{\omega}^t_m) \mathbf{p}_m^t \end{pmatrix}^T \\
\mathbf{g}^{t+1} = \mathbf{W}^t \left(\mathbf g^t+G(\bm\omega^t)-G(\bm\omega^{t-1})\right) \\
\aleph^{t+1} = \mathbf{W}^t \left(\mathbf \aleph^t+H(\bm\omega^t)-H(\bm\omega^{t-1})\right)
\\
\bm\omega^{t+1} = \mathbf{W}^t\left(\mathcal{M}(\bm\omega^t+\epsilon(\mathbf{p}^t+\epsilon \bm g^{t+1}),\bm\omega^t,\aleph^{t+1},u^t)\right) \\
\end{array}
\end{equation}
where recall that $\mathbf{p}^t$ be a normal random variable and $u^t$ is uniformly distributed and by $\mathcal{M}$ we denote the approximate metropolis acceptance operator, defined to be,
\[
\mathcal{M}(\bm{\omega},\bm{\omega}',\aleph,u)_i = \left\{ 
\begin{array}{lr}
\bm{\omega}_i & \text{if }\rho(\aleph_i,\bm{g}_i,\bm{\omega}_i)\ge \log u\\
\bm{\omega}'_i & \text{otherwise}
\end{array}
\right.
\]

We prepare a coupling argument for the convergence (see, for example, e.g.,~\cite{durmus2019high}). We bound the distance in probability measure of the the chain generated by~\eqref{eq:hmcvecstackctaalg} to the chain generated by $m$ parallel centralized HMC chains all with access to the entire dataset, whose ergodicity properties are well established, for instance in~\cite{durmus2017convergence}. We show this by coupling~\eqref{eq:hmcvecstackctaalg} to centralized HMC, using sets of intermediate quantities. As such we consider the three chains with the exogenous random variables, $\mathbf{p}^t$ and $u^t$ as the same across all of the intermediate constructed chains for all $t$.

We now indicate $\bar{a}$ to indicate the averaged and copied $m$ times vector $\{a_i\}$, e.g., 
\[
\bar{\bm{\omega}}^t = \begin{pmatrix} \frac{1}{m}\sum_i \left(\bm{\omega}^t_i\right)^T
& \frac{1}{m}\sum_i \left(\bm{\omega}^t_i\right)^T & ... & \frac{1}{m}\sum_i \left(\bm{\omega}^t_i\right)^T \end{pmatrix}^T
\]
and similarly for $\bm g^t$ and $\bar\aleph^t$. The vector recursion for the averaged iterate satisfies,
\begin{equation}\label{eq:hmcvecstackavg}
\begin{array}{l}
\bar{\mathbf{g}}^{t+1} = \frac{1}{m}(\bm{I}\otimes\bm{1} \bm{1}^T)\left(\bar{\bm{g}}^t+G(\bm{\omega}^t)-G(\bm{\omega}^{t-1})\right) \\
\bar{\mathbf{\aleph}}^{t+1} = \frac{1}{m}(\bm{I}\otimes\bm{1} \bm{1}^T)\left(\bar \aleph^t+H( \bm{\omega}^t)-H(\bm{\omega}^{t-1})\right) \\
\bar{\bm{\omega}}^{t+1} =\frac{1}{m}(\bm{I}\otimes\bm{1} \bm{1}^T) \left(\mathcal{M}(\bm{\omega}^t+\epsilon(\bm{p}^t+\epsilon \bm{g}^{t+1}),\bm{\omega}^t,\aleph^{t+1},u^t)\right) \\
\end{array}
\end{equation}

Now we consider a hypothetical chain wherein there is a stack of $m$ vectors undergoing the decentralized HMC iteration, and subsequently averaged and dispersed across the stack of $m$.
In effect this is the chain wherein the evaluations are performed at the average parameter, and corresponds to exact HMC, however, with the approximate Metropolis operator using the second order approximation.
\begin{equation}\label{eq:hmcvecstackavgpure}
\begin{array}{l}
\tilde{\mathbf{g}}^{t+1} = \frac{1}{m}(\bm{I}\otimes\bm{1} \bm{1}^T)\left(\tilde{\bm{g}}^t+G(\tilde{ \bm{\omega}}^t)-G(\tilde{\bm{\omega}}^{t-1})\right) \\
\bar{\mathbf{\aleph}}^{t+1} = \frac{1}{m}(\bm{I}\otimes\bm{1} \bm{1}^T)\left(\tilde \aleph^t+H(\tilde{ \bm{\omega}}^t)-H(\tilde{\bm{\omega}}^{t-1})\right) \\
\tilde{\bm{\omega}}^{t+1} =\frac{1}{m}(\bm{I}\otimes\bm{1} \bm{1}^T) \left(\mathcal{M}(\tilde{\bm{\omega}}^t+\epsilon(\bm{p}^t+\epsilon \tilde{\bm{g}}^{t+1}),\tilde{\bm{\omega}}^t,\tilde\aleph^{t+1},u^t)\right) \\
\end{array}
\end{equation}
Finally, we write that the centralized HMC iteration as, noting that technically it is a stack of $m$ copies of centralized HMC runs,
for formal simplicity averaged together (in effect reducing the variance),
\begin{equation}\label{eq:centralhmc}
    \hat{\bm{\omega}}^{t+1} = \frac{1}{m}(\bm{I}\otimes\bm{1} \bm{1}^T)\hat{\mathcal{M}}(\hat{\bm{\omega}}^t+\epsilon(\bm{p}^t+\epsilon \bar{G}(\hat{\bm{\omega}}^t)),\hat{\bm{\omega}}^t,u^t)
\end{equation}
where now the Metropolis operator $\hat{\mathcal{M}}$ uses the exact values of the pdf $p(\bm{\omega}^t)$ to calculate acceptance and rejection as opposed to an approximation associated with $\aleph^t$, and $\bar{G}$ is defined as,
\[
\bar{G}(\bm{\omega}) = \begin{pmatrix} \sum_i \nabla_{\bm{\omega}} p(\bm{\omega}|X_i,Y_i)^T &
\sum_i \nabla_{\bm{\omega}} p(\bm{\omega}|X_i,Y_i)^T & ... & \sum_i \nabla_{\bm{\omega}} p(\bm{\omega}|X_i,Y_i)^T \end{pmatrix}^T
\]

We make the following assumption on the potential function $p(\bm{\omega})$.
\begin{assumption}
It holds that $p(\bm{\omega})$ as a function of $\bm{\omega}$ is two-times Lipschitz continuously differentiable with constants $L_2=\sup_{\bm{\omega}}\|\nabla^2_{\bm{\omega}\bm{\omega}} p(\bm{\omega})\|$ and $L_3=\sup_{\bm{\omega}}\|\nabla^3_{\bm{\omega}\bm{\omega}\bm{\omega}} p(\bm{\omega})\|$.
\end{assumption}

The bound between the approximate and exact HMC is stated first, with the error accumulating from the Taylor remainder error of the Metropolis acceptance step. Note that the scale of the error is of order $O(\epsilon^{5})$ in the step-size.
\begin{theorem}\label{th:approxexact}
The $L2$ distance between $\tilde{\bm{\omega}}^t$ and $\hat{\bm{\omega}}^t$ is bounded up to iteration $t$ by the following expression,
\begin{equation}
    \mathbb{E}\left\|\tilde{\bm{\omega}}^{t+1}-\hat{\bm{\omega}}^{t+1}\right\|^2 \le \sum\limits_{s=0}^t \left[1+(L^2_2 \epsilon^4+1)\left(\frac{4\epsilon^3\mathbb{E}\|\bm{p}\|^3 L_3}{3}\right)\right]^s 
    \left(\frac{4\epsilon^5\mathbb{E}\|\bm{p}\|^3 L_3}{3}\right)\left[\mathbb{E}\|\bm{p}^t\|^2+\epsilon^2 U\right]
\end{equation}
where $U$ is the expected gradient norm bound on centralized HMC applied to the problem, which exists due to ergodicity.
\end{theorem}

Next we indicate the \emph{consensus error}, the standard measure for the discrepancy between the parameter estimates across the agents. The consensus error \emph{does not} accumulate, but rather stays upper bounded, with the mean-squared-error scaling with $O(\epsilon^2)$.

\begin{theorem}\label{th:consensus}
The consensus error satisfies the following $L2$ expectation bound:
\begin{equation}\label{eq:conserrorboundall}
\begin{array}{l}
    \mathbb{E}\|\bar{\bm{\omega}}^{t}-\bm{\omega}^{t}\| +
\mathbb{E}\|\bar{\bm{g}}^{t}-\bm{g}^{t}\|+ \mathbb{E}\|\bar\aleph^{t}-\aleph^{t}\| \le 
\frac{\epsilon \hat C \left\|\mathbf{M}\right\|}{1-\beta} \\
    \mathbb{E}\|\bar{\bm{\omega}}^{t}-\bm{\omega}^{t}\|^2 +
\mathbb{E}\|\bar{\bm{g}}^{t}-\bm{g}^{t}\|^2+ \mathbb{E}\|\bar\aleph^{t}-\aleph^{t}\|^2 \le 
\frac{\epsilon^2 \hat C \left\|\mathbf{M}\right\|^2}{1-\beta^2} 
\end{array}
\end{equation}
for some $\hat C>0$ depending on $L_2$, $L_3$ and $\epsilon$.
\end{theorem}

Finally, the most involved derivation is the probabilistic mass error between approximate HMC and the evolution of the average of the iterates, with the discrepancy accumulating due to the evaluation of the gradient vectors at the parameter values taking place at individual agents' parameter estimates rather than at the average. Formally,
\begin{theorem}\label{th:approxavg}
The $L2$ expected error accumulates as,
\begin{equation}
    \mathbb{E}\|\tilde{\bm{\omega}}^t-\bar{\bm{\omega}}^t\|+
    \mathbb{E}\|\tilde{\bm{g}}^t-\bar{\bm{g}}^t\|+\mathbb{E}\|\tilde{\aleph}^t-\bar{\aleph}^t\| \le \epsilon^3 \tilde{C}^{2t} B
\end{equation}
where $\tilde C>1$ and depends on $L_2$ and $\epsilon$, while $B$ depends on $L_2$, $\epsilon$, $\beta$ and $\|\mathbf{M}\|$.
\end{theorem}

Considering Theorems~\ref{th:approxexact},~\ref{th:consensus} and~\ref{th:approxavg}, we have that there is a stable constant error proportional to $\epsilon^2$, as according to the consensus error, and an error that accumulates exponentially with the number of iterations and scales with $\epsilon^3$. Thus, broadly speaking, the decentralized HMC Algorithm as presented in Algorithm~\ref{alg:mainalg} manages to recreate the behavior of centralized HMC up to some error.
Given that HMC exhibits ergodicity towards the stationary distribution in the general case, and geometric ergodicity for strongly log-concave distributions~\cite{durmus2017convergence},  our approach generates a sequence of samples with controlled approximate accuracy relative to one with these properties.

\section{Examples}\label{s:ex}
In this section we illustrate the performance of Algorithm~\ref{alg:mainalg}, which we shall refer to as \textbf{DecentralizedMALA}, comparing it to two baselines, \textbf{CentralizedHMC}, which is performing HMC on the entire data set, and \textbf{DecentralizedULA}, which we implement as in~\cite{parayil2020decentralized}.\footnote{However, unlike in \cite{parayil2020decentralized}, where they use stochastic gradients we apply their method without taking stochastic gradients.} We leave the discussion of the computing platform, problem dimensions, and hyperparameter selection to the appendix.
\subsection{Gaussian Mixture Model}
The purpose of this example is to illustrate pictorially the importance of the Metropolis adjustment step. Here we sample from a Gaussian mixture given by 
$$\theta_1 \sim \mathcal{N}(0,\sigma_1^2);\ \ \theta_2 \sim \mathcal{N}(0,\sigma_2^2) \text{ and } x_i \sim \left(\frac{1}{2} \mathcal{N}(\theta_1,\sigma_x^2) + \frac{1}{2} \mathcal{N}(\theta_1 + \theta_2,\sigma_x^2)\right)$$
where $\sigma_1^2 = 10, \sigma_2^2 = 1, \sigma_x^2 = 2$. For the centralized setting, $100$ data samples are drawn from the model with $\theta_1 = 0$ and $\theta_2 = 1$, These are then split into $5$ sets of 20 samples that are made available to each agent in the decentralised network. This is a similar setting to that of \cite{parayil2020decentralized}. Figure \ref{fig:GaussianMM} compares sampling from the posterior distribution in both the centralized and decentralized settings. The contour plot corresponds to the true value of the log potential. The plot in the first column displays the samples from the centralized approach and can be thought of as the ``ground truth". Columns two and three display the materialized samples for the decentralized setting, where column two applies the Taylor approximated Metropolis adjustment and column three does not. The discrepancy between these two schemes can be seen qualitatively via the spread of the samples. The Metropolis adjustment prevents the collection of samples in low probability regions and ensures that the samples stay in the high probability region in a similar manner to the centralized approach. Leaving out the Metropolis step means that samples that have low log probability are never rejected.


\begin{figure}[h!]
  \includegraphics[width=\linewidth]{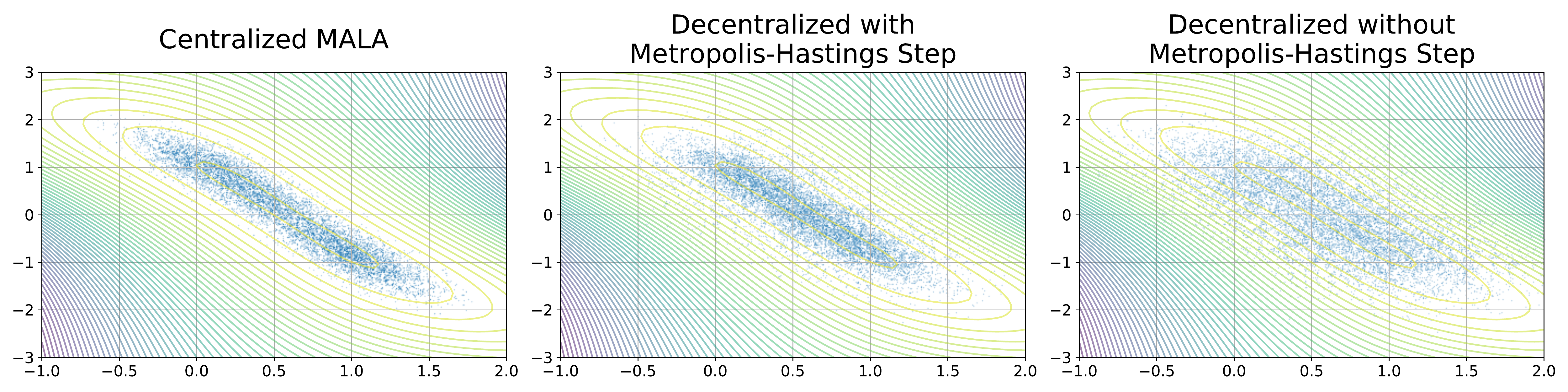}
  \caption{Gaussian mixture model.}
  \label{fig:GaussianMM}
\end{figure}


\subsection{Linear Regression}
\begin{wrapfigure}{r}{0.55\textwidth}
  \begin{center}
    \includegraphics[width=0.5\textwidth]{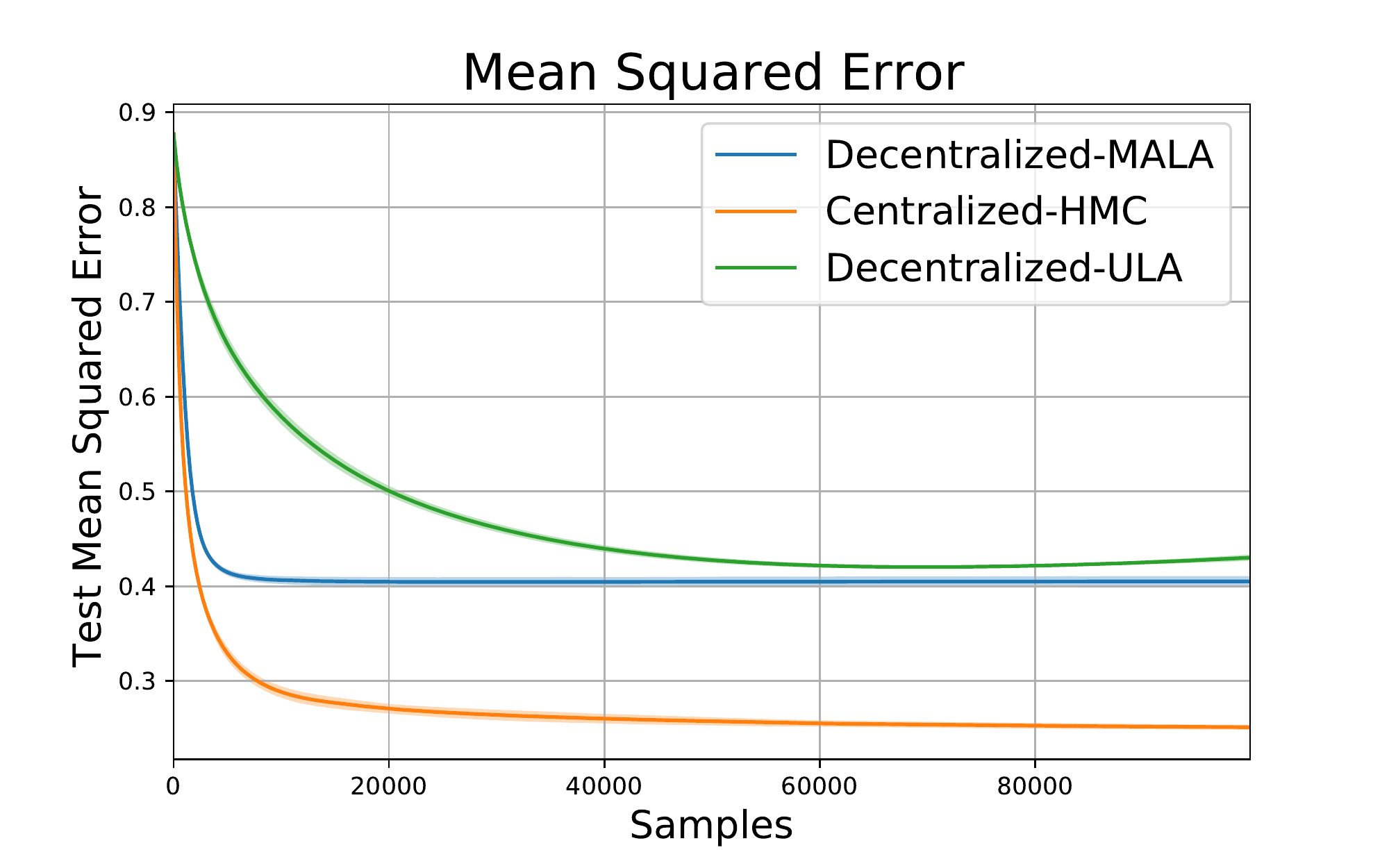}
  \end{center}
  \caption{Mean Squared Error by Samples for Linear regression. The Decentralized MALA}\label{fig:lin_reg}
\end{wrapfigure}
For our first example with real data, we investigate Bayesian linear regression with four agents applied to the Boston Housing data set \cite{harrison1978hedonic}. In the decentralized setting, each agent is only given access to separate parts of the 13-dimensional feature space. Agents 1-3 have access to 3 input features (each a different set of 3) and agent 4 sees the remaining 4 features. We use a simple normal prior for each parameter of the model and compare the results in Figure \ref{fig:lin_reg}, which
displays the cumulative mean squared error over samples. Our approach (in blue) converges and outperforms the baseline decentralized approach. 
The centralized approach, with access to all features achieves the best performance.


\subsection{Logistic Regression}

For logistic regression we work with the MNIST data set \cite{lecun1998gradient}. We define two scenarios for decentralized learning. In scenario one, we only allow each agent partial observation of the total training set. In scenario two, we distribute the agents in a ring network such that each agent can only communicate with its two neighbors. Furthermore, for the ring network, each agent only has access to two classes of the training data. 

\textbf{Partial Observation.} For this experiment there are four agents where each only sees one quarter of the MNIST digit figure. The rest of the feature space is set to zeros. Figure \ref{fig:log_partial} shows the results in statistical performance for the three schemes. The plots show cumulative performance averaged over all agents. The confidence intervals are one standard deviation and are calculated over 9 random seeds for all approaches. 

\begin{figure}[h!]
\centering
  \includegraphics[width=0.7\linewidth]{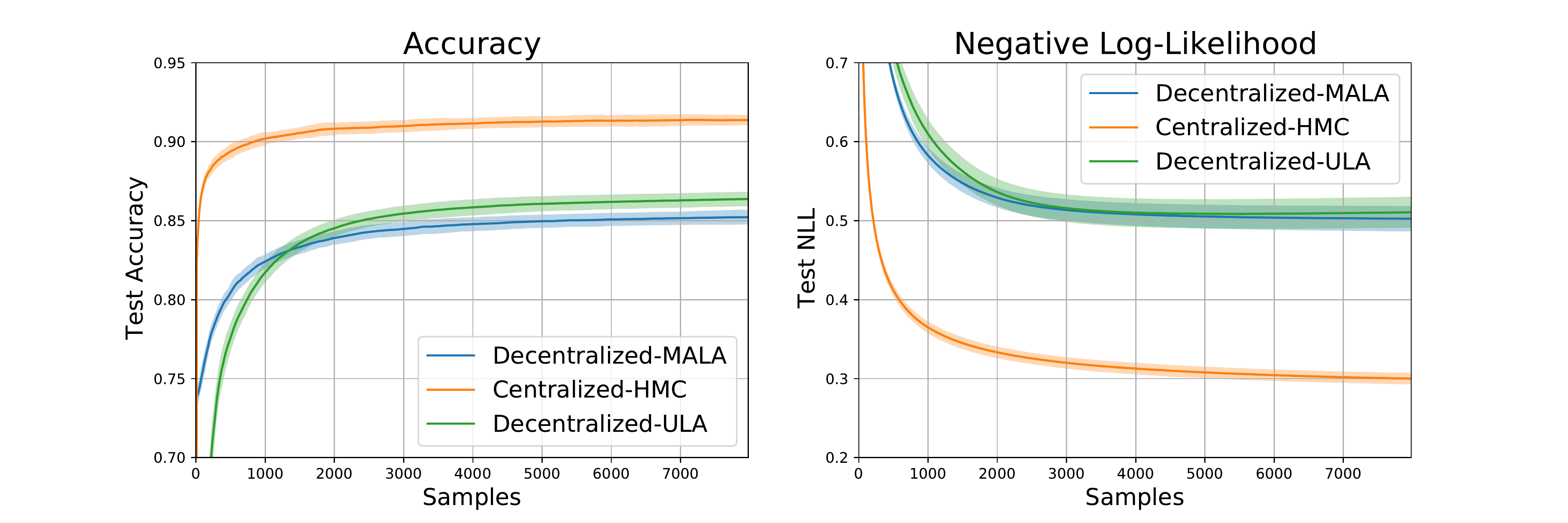}
  \caption{Partial MNIST Set-up Logistic Regression. The two Decentralized algorithms perform similarly, each obtaining an overall worse estimate than the centralized baseline.}
  \label{fig:log_partial}
\end{figure}

\textbf{Ring Network.} For this experiment there are five agents connected in a ring formation, such that each agent can only communicate with two other agents. Each agent only has access to two classes of digits in their training data, e.g. agent 0's training data consists of digits 0-1. Figure \ref{fig:log_ring} displays that in this setting, \textbf{DecentralizedMALA} achieves similar performance to \textbf{CentralizedHMC}, both of which in turn outperform \textbf{DecentralizedULA}.

\begin{figure}[h!]
\centering
  \includegraphics[width=0.7\linewidth]{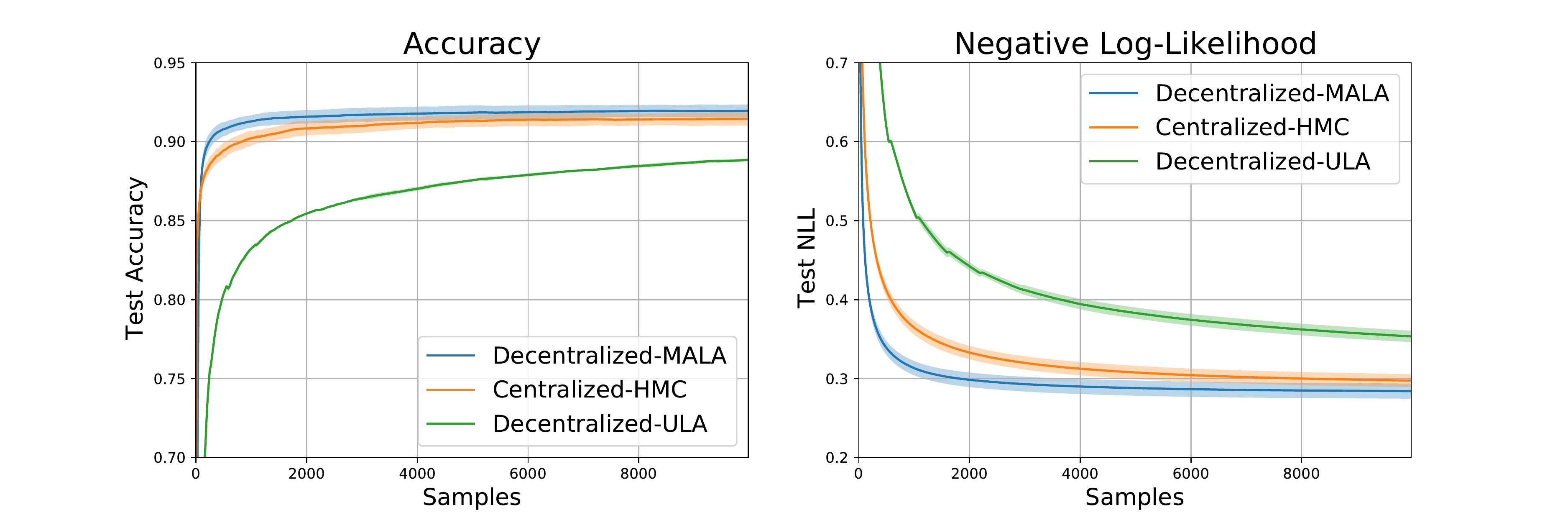}
  \caption{Ring Set-up Logistic Regression. \textbf{DecentralizedMALA} performs just as well as its centralized counterpart, while outperforming the Langevin-based approach.}
  \label{fig:log_ring}
\end{figure}


\subsection{Bayesian Neural Networks}

A unique contribution of our proposed algorithm is its treatment of nonconvex functions which has not been addressed in earlier work on decentralized sampling \cite{kungurtsev2020stochastic,gurbuzbalaban2020decentralized,parayil2020decentralized}. To study this class of potentials, we ran our decentralized sampling scheme over two agents, each with their own neural network model and data. We provide each agent with half the classes available in the MNIST data set, i.e. one with access to digits 0-4 and the other with access to digits 5-9. Each network is fully fully connected with a single hidden layer of $100$ units.
Figure \ref{fig:bnn_split} displays the cumulative performance over the number of samples, averaged across the two agents for the test data. As is evident from the figure, our approach is able to learn in a decentralized fashion over nonconvex models. Without explicitly passing data between the two agents, each agent achieves good performance over classes it has never seen before.

\begin{figure}[h!]
\centering
  \includegraphics[width=0.7\linewidth]{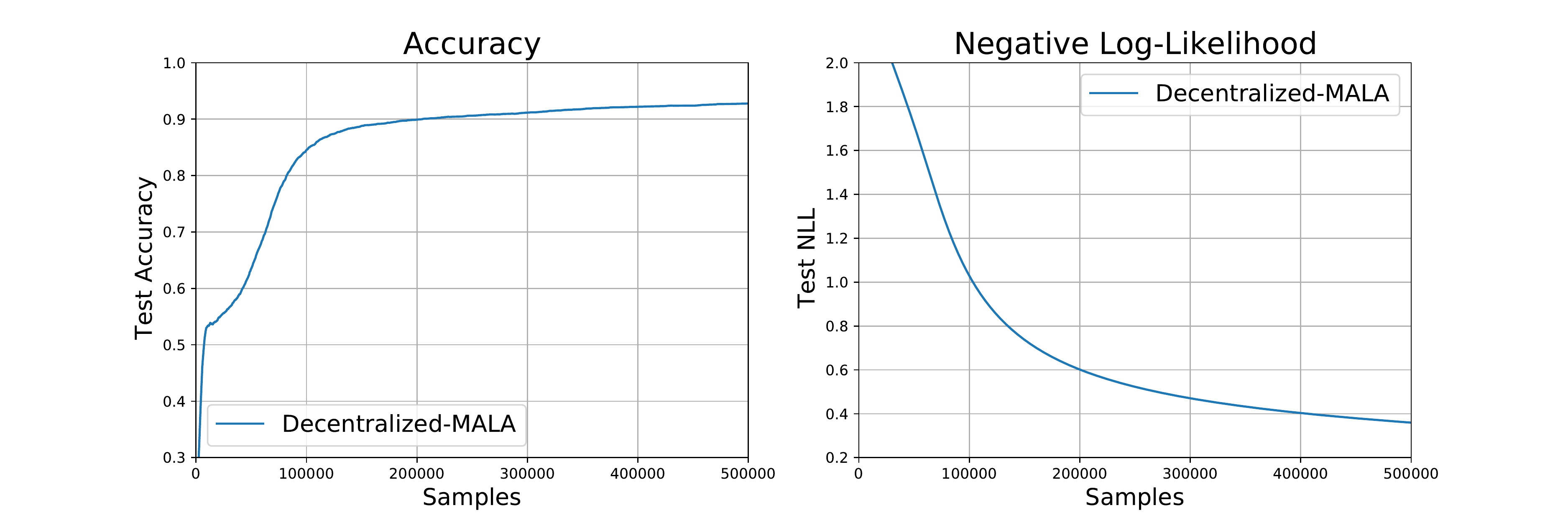}
  \caption{BNNs: Two Agents, one with digits 0-4 and the other with digits 5-9. Test data contains all classes and performance is average over both agents.}
  \label{fig:bnn_split}
\end{figure}

\section{Perspectives and Conclusion}

We presented an Algorithm that performs Metropolis-adjusted HMC sampling in a decentralized environment. Theoretically it appears to behave close in probability to exact HMC sampling, and numerically, appears to perform well on standard datasets. Theoretically, it would be interesting to establish ergodicity and convergence in probability of the chain generated by the Algorithm itself, rather than just a quantitative distance to the ergodic chain -- whether this is possible is still an open question.

\bibliographystyle{plain}
\bibliography{refs}


\appendix

\section{Proofs of Theoretical Results}

\subsection{Coupling Approximate to Centralized HMC}
We first bound the distance in probability for the chain as governing the update for $\hat{\bm{\omega}}^t$ in~\eqref{eq:centralhmc} and
for $\bar{\bm{\omega}}^t$ in~\eqref{eq:hmcvecstackavg}.

Note that by construction, it always holds that,
\begin{equation}\label{eq:centralgaleph}
\tilde{\bm{g}}^t = \mathbf{1}_m\otimes \sum_i \nabla p(\tilde{\bm{\omega}},X_i,Y_i)=\bar{G}(\tilde{\bm{\omega}})\text{ and }\tilde{\aleph}^t = 
\mathbf{1}_m\otimes \sum_i (\mathbf{p}_i^t)^T\left(\nabla^2_{\bm{\omega}^2} p(\tilde{\bm{\omega}},X_i,Y_i)\right)(\mathbf{p}_i^t) 
\end{equation}

Thus the only cause of a discrepancy between the chains for $\tilde{\bm{\omega}}^t$ and $\hat{\bm{\omega}}^t$ is the truncation of the potential at the second order to compute the acceptance probability.
In particular we know that the error in this case is bounded by the Taylor expansion error, which is bounded by,
\begin{equation}\label{eq:alepherr}
\left|\frac{\epsilon^3}{6} \frac{\partial^3 U}{\partial \bm{\omega}^3} [\bm{p}^t][\bm{p}^t][\bm{p}^t]\right| \le \frac{\epsilon^3 L_3 \|\bm{p}^t\|^3}{6}
\end{equation}
where $L_3$ is the Lipschitz constant for the second derivative of the potential function. 

Thus the discrepancy between $\hat{\bm{\omega}}$ and $\tilde{\bm{\omega}}$
amounts to the possibility of acceptance in one case and not the other, whose probability is bounded by~\eqref{eq:alepherr} with the error being bounded by the change in the step, or $\epsilon(\bm{p}^t+\epsilon \bar{G}(\tilde{\bm{\omega}}^t))$ (and $\bar{G}(\hat{\bm{\omega}})$ in the other case). Let us write this formally,
\begin{theorem}
The $L2$ distance between $\tilde{\bm{\omega}}^t$ and $\hat{\bm{\omega}}^t$ is bounded up to iteration $t$ by the following expression,
\begin{equation}
    \mathbb{E}\left\|\tilde{\bm{\omega}}^{t+1}-\hat{\bm{\omega}}^{t+1}\right\|^2 \le \sum\limits_{s=0}^t \left[1+(L^2_2 \epsilon^4+1)\left(\frac{4\epsilon^3\mathbb{E}\|\bm{p}\|^3 L_3}{3}\right)\right]^s 
    \left(\frac{4\epsilon^5\mathbb{E}\|\bm{p}\|^3 L_3}{3}\right)\left[\mathbb{E}\|\bm{p}^t\|^2+\epsilon^2 U\right]
\end{equation}
\end{theorem}
\begin{proof}

Thus we have that,
\[
\begin{array}{l}
\mathbb{E}\left\|\tilde{\bm{\omega}}^{t+1}-\hat{\bm{\omega}}^{t+1}\right\|^2
\le \mathbb{E}\left\|\tilde{\bm{\omega}}^{t}-\hat{\bm{\omega}}^{t}\right\|^2 \\ 
\qquad+\mathbb{P}\left[(\mathcal{M}(\tilde{\bm{\omega}}^{*},\tilde{\bm{\omega}}^{t},\tilde{\aleph}^{t+1},u^t) = \tilde{\bm{\omega}}^{*}) \cap (\hat{\mathcal{M}}(\hat{\bm{\omega}}^{*},\hat{\bm{\omega}}^{t},u^t) = \hat{\bm{\omega}}^{t})\right]
\mathbb{E}\left\|\epsilon(\bm{p}^t+\epsilon \bar{G}(\tilde{\bm{\omega}}^t))+\tilde{\bm{\omega}}^t-\hat{\bm{\omega}}^t\right\|^2
\\
\qquad +\mathbb{P}\left[(\mathcal{M}(\tilde{\bm{\omega}}^{*},\tilde{\bm{\omega}}^{t},\tilde{\aleph}^{t+1},u^t) = \tilde{\bm{\omega}}^{t}) \cap (\hat{\mathcal{M}}(\hat{\bm{\omega}}^{*},\hat{\bm{\omega}}^{t},u^t) = \hat{\bm{\omega}}^{*})\right] 
\mathbb{E}\left\|\epsilon(\bm{p}^t+\epsilon \bar{G}(\hat{\bm{\omega}}^t))+\hat{\bm{\omega}}^t-\tilde{\bm{\omega}}^t\right\|^2 
\\
\le \mathbb{E}\left\|\tilde{\bm{\omega}}^{t}-\hat{\bm{\omega}}^{t}\right\|^2+
\left(\frac{\epsilon^3 L_3 \mathbb{E}\|\bm{p}^t\|^2}{3}\right)\left[\mathbb{E}\left\|\tilde{\bm{\omega}}^{t}-\hat{\bm{\omega}}^{t}\right\|^2+\epsilon^2\mathbb{E}\left\|\bm{p}^t\right\|^2+\epsilon^4\mathbb{E}\left\|\bar{G}(\tilde{\bm{\omega}}^t))\right\|^2\right]
\\
\qquad\qquad+
\left(\frac{\epsilon^3 L_3 \mathbb{E}\|\bm{p}^t\|^2}{3}\right)\left[\mathbb{E}\left\|\tilde{\bm{\omega}}^{t}-\hat{\bm{\omega}}^{t}\right\|^2+\epsilon^2\mathbb{E}\left\|\bm{p}^t\right\|^2+\epsilon^4\mathbb{E}\left\|\bar{G}(\hat{\bm{\omega}}^t))\right\|^2\right]
\\
\le \mathbb{E}\left\|\tilde{\bm{\omega}}^{t}-\hat{\bm{\omega}}^{t}\right\|^2+
\left(\frac{2\epsilon^3 L_3 \mathbb{E}\|\bm{p}^t\|^2}{3}\right)\left[\mathbb{E}\left\|\tilde{\bm{\omega}}^{t}-\hat{\bm{\omega}}^{t}\right\|^2+\epsilon^2\mathbb{E}\left\|\bm{p}^t\right\|^2\right]
\\
\qquad\qquad+
\left(\frac{4\epsilon^7 L_3 \mathbb{E}\|\bm{p}^t\|^2}{3}\right)\left[
\mathbb{E}\left\|\bar{G}(\hat{\bm{\omega}}^t))\right\|^2+
L_2^2\mathbb{E}\left\|\tilde{\bm{\omega}}^{t}-\hat{\bm{\omega}}^{t}\right\|^2
\right]
\\
\le \left[1+(L^2_2 \epsilon^4+1)\left(\frac{4\epsilon^3\mathbb{E}\|\bm{p}\|^3 L_3}{3}\right)\right]\mathbb{E}\left\|\tilde{\bm{\omega}}^{t}-\hat{\bm{\omega}}^{t}\right\|^2
\\ \qquad\qquad + \left(\frac{4\epsilon^5\mathbb{E}\|\bm{p}\|^3 L_3}{3}\right)
\left[\mathbb{E}\|\bm{p}^t\|^2+\epsilon^2 U\right]
\end{array}
\]
where $U$ is a bound for $\mathbb{E}\|\hat{\bm{\omega}}^t\|$
which exists by the ergodicity of~\eqref{eq:centralhmc}.

Thus, with $\tilde{\bm{\omega}}^0=\hat{\bm{\omega}}^0$ we have that,
\begin{equation}
    \mathbb{E}\left\|\tilde{\bm{\omega}}^{t+1}-\hat{\bm{\omega}}^{t+1}\right\|^2 \le \sum\limits_{s=0}^t \left[1+(L^2_2 \epsilon^4+1)\left(\frac{4\epsilon^3\mathbb{E}\|\bm{p}\|^3 L_3}{3}\right)\right]^s 
    \left(\frac{4\epsilon^5\mathbb{E}\|\bm{p}\|^3 L_3}{3}\right)\left[\mathbb{E}\|\bm{p}^t\|^2+\epsilon^2 U\right]
\end{equation}
\end{proof}

\subsection{Consensus between Decentralized and Averaged HMC}
Now we relate the process as generated by Algorithm~\ref{alg:mainalg} to the average dynamics as given by~\eqref{eq:hmcvecstackavg}. 

\begin{theorem}
The consensus error satisfies the following $L2$ expectation bound:
\begin{equation}\label{eq:conserrorboundaa}
    \mathbb{E}\|\bar{\bm{\omega}}^{t}-\bm{\omega}^{t}\|^2 +
\mathbb{E}\|\bar{\bm{g}}^{t}-\bm{g}^{t}\|^2+ \mathbb{E}\|\bar\aleph^{t}-\aleph^{t}\|^2 \le 
\frac{\epsilon^2 \hat C \left\|\mathbf{M}\right\|^2}{1-\beta^2}
\end{equation}
for some $\hat C>0$ depending on $L_2$, $L_3$ and $\epsilon$, and by a similar reasoning, using redundant notation,
\begin{equation}\label{eq:conserrorboundab}
    \mathbb{E}\|\bar{\bm{\omega}}^{t}-\bm{\omega}^{t}\| +
\mathbb{E}\|\bar{\bm{g}}^{t}-\bm{g}^{t}\|+ \mathbb{E}\|\bar\aleph^{t}-\aleph^{t}\| \le 
\frac{\epsilon \hat C \left\|\mathbf{M}\right\|}{1-\beta}
\end{equation}

\end{theorem}
\begin{proof}
Consider the recursion in expected $L2$ error. 

\begin{equation}\label{eq:coupledecavg}
\begin{array}{l}
\mathbb{E}\|\bar{\bm{\omega}}^{t+1}-\bm{\omega}^{t+1}\|^2 +
\mathbb{E}\|\bar{\bm{g}}^{t+1}-\bm{g}^{t+1}\|^2+ \mathbb{E}\|\bar\aleph^{t+1}-\aleph^{t+1}\|^2 \\ 
\le \mathbb{E}\left\| \mathbf{W}^t\mathcal{M}({\bm{\omega}}^t+\epsilon(\bm{p}^t+\epsilon \bm{g}^{t+1}),\bm{\omega}^t,\aleph^{t+1})-\frac{1}{m}(\bm{I}\otimes\bm{1} \bm{1}^T)\mathcal{M}({\bm{ \omega}}^t+\epsilon(\bm{p}^t+\epsilon {\bm{ g}}^{t+1}),{\bm{\omega}}^t,\aleph^{t+1})\right\|^2 \\
\quad + 2\mathbb{E}\left\|\left(\mathbf{W}^t-\frac{1}{m}(\bm{I}\otimes\bm{1} \bm{1}^T)\right) \left(\bar{\bm{g}}^t-\bm{g}^t\right) \right\|^2\\ \quad+2\mathbb{E}\left\|\left(\mathbf{W}^t-\frac{1}{m}(\bm{I}\otimes\bm{1} \bm{1}^T)\right) \left[G( \bm{\omega}^t)-G(\bar {\bm{\omega}}^t)-G( \bm{\omega}^{t-1})+G(\bar{\bm{\omega}}^{t-1})\right]\right\|^2 \\ 
\quad + 2\mathbb{E}\left\|\left(\mathbf{W}^t-\frac{1}{m}(\bm{I}\otimes\bm{1} \bm{1}^T)\right) \left(\bar \aleph^t-\aleph^t\right) \right\|^2\\ \quad +2\mathbb{E}\left\|\left(\mathbf{W}^t-\frac{1}{m}(\bm{I}\otimes\bm{1} \bm{1}^T)\right) \left[H(\bm{\omega}^t)-H(\bar{\bm{\omega}}^t)-H(\bm{\omega}^{t-1})+H(\bar{\bm{\omega}}^{t-1})\right]\right\|^2 \\ \le 2(1+\epsilon^4)
\mathbb{E}\left\|\left(\mathbf{W}^t-\frac{1}{m}(\bm{I}\otimes\bm{1} \bm{1}^T)\right) \left(\bar{\bm{g}}^t-\bm{g}^t\right)\right\|^2+\epsilon^2 \mathbb{E}\left\|\left(\mathbf{W}^t-\frac{1}{m}(\bm{I}\otimes\bm{1} \bm{1}^T)\right) \bm{p}^t\right\|^2 \\
\quad +2(1+L_2+L_3)\mathbb{E}\left\|\left(\mathbf{W}^t-\frac{1}{m}(\bm{I}\otimes\bm{1} \bm{1}^T)\right) \left(\bar{\bm{\omega}}^t-\bm{\omega}^t\right)\right\|^2\\ \quad +2(L_2+L_3)\mathbb{E}\left\|\left(\mathbf{W}^t-\frac{1}{m}(\bm{I}\otimes\bm{1} \bm{1}^T)\right) \left(\bar{\bm{\omega}}^{t-1}-\bm{\omega}^{t-1}\right)\right\|^2 \\
\quad + 2\mathbb{E}\left\|\left(\mathbf{W}^t-\frac{1}{m}(\bm{I}\otimes\bm{1} \bm{1}^T)\right) \left(\bar \aleph^t-\aleph^t\right) \right\|^2 \\ 
\le 2(1+\epsilon^4)
\beta^{2t} \mathbb{E}\left\|\bar{\bm{g}}^t-\bm{g}^t\right\|^2+\epsilon^2 \beta^{2t}\mathbb{E}\left\| \bm{p}^t\right\|^2 \\
\quad +2(1+L_2+L_3)\beta^{2t}\mathbb{E}\left\|\bar{\bm{\omega}}^t-\bm{\omega}^t\right\|^2+2(L_2+L_3)\beta^{2t}\mathbb{E}\left\| \left(\bar{\bm{\omega}}^{t-1}-\bm{\omega}^{t-1}\right)\right\|^2 \\
\quad + 2\beta^{2t}\mathbb{E}\left\|\bar \aleph^t-\aleph^t\right\|^2
\end{array}
\end{equation}
where we have used that $\mathbf{W}^t \bar{\bm{g}}^t = \frac{1}{m}\left(\bm{I}\otimes\bm{1} \bm{1}^T\right) \bm{g}^t= \frac{1}{m}\left(\bm{I}\otimes\bm{1} \bm{1}^T\right) \bar{\bm{g}}^t$, etc. throughout and, e.g.,\cite[Lemma 6]{di2016next} for the fact that $\left\|\left(\mathbf{W}^t-\frac{1}{m}(\bm{I}\otimes\bm{1} \bm{1}^T)\right)\right\|\le \beta^t$.
\end{proof}


\subsection{Coupling Averaged to Approximate HMC}
Finally we derive the most involved expression, given by the Theorem, restated,
\begin{theorem}\label{th:approxavgb}
The $L2$ expected error accumulates as,
\begin{equation}
    \mathbb{E}\|\tilde{\bm{\omega}}^t-\bar{\bm{\omega}}^t\|+
    \mathbb{E}\|\tilde{\bm{g}}^t-\bar{\bm{g}}^t\|+\mathbb{E}\|\tilde{\aleph}^t-\bar{\aleph}^t\| \le \epsilon^3 \tilde{C}^{2t} B
\end{equation}
where $\tilde C>1$ and depends on $L_2$ and $\epsilon$, while $B$ depends on $L_2$, $\epsilon$, $\beta$ and $\|\mathbf{M}\|$.
\end{theorem}
\begin{proof}

We have,
\begin{equation}\label{eq:avgapproxg}
\begin{array}{l}
\mathbb{E}\left\|\tilde{\bm{g}}^{t+1}-\bar{\bm{g}}^{t+1}\right\| \le \\ \mathbb{E}\left\|\frac{1}{m}(\bm{I}\otimes\bm{1} \bm{1}^T)\left[\tilde{\bm{g}}^{t}-\bar{\bm{g}}^t+G(\tilde{\bm{\omega}}^t)-G(\bar{\bm{\omega}}^t)+G(\bar{\bm{\omega}}^t)-G({\bm{\omega}}^t)+G(\tilde{\bm{\omega}}^{t-1})-G(\bar{\bm{\omega}}^{t-1})+G(\bar{\bm{\omega}}^{t-1})-G({\bm{\omega}}^{t-1})\right]\right\| \\ \le
\mathbb{E}\left\|\tilde{\bm{g}}^{t}-\bar{\bm{g}}^t\right\|+
L_2 \mathbb{E}\left\|\tilde{\bm{\omega}}^{t}-\bar{\bm{\omega}}^t\right\|+L_2\mathbb{E} \left\|{\bm{\omega}}^{t}-\bar{\bm{\omega}}^t\right\| +
L_2 \mathbb{E}\left\|\tilde{\bm{\omega}}^{t-1}-\bar{\bm{\omega}}^{t-1}\right\|+L_2 \mathbb{E}\left\|{\bm{\omega}}^{t-1}-\bar{\bm{\omega}}^{t-1}\right\| 
\end{array}
\end{equation}
By the same argument we have,
\begin{equation}\label{eq:avgapproxal}
\begin{array}{l}
\mathbb{E}\left\|\tilde{{\aleph}}^{t+1}-\bar{{\aleph}}^{t+1}\right\| \le \mathbb{E}\left\|\tilde{{\aleph}}^{t}-\bar{{\aleph}}^t\right\|+
L_2\mathbb{E} \left\|\tilde{\bm{\omega}}^{t}-\bar{\bm{\omega}}^t\right\|+L_2 \mathbb{E}\left\|{\bm{\omega}}^{t}-\bar{\bm{\omega}}^t\right\| \\ \qquad\qquad\qquad+
L_2\mathbb{E} \left\|\tilde{\bm{\omega}}^{t-1}-\bar{\bm{\omega}}^{t-1}\right\|+L_2 \mathbb{E} \left\|{\bm{\omega}}^{t-1}-\bar{\bm{\omega}}^{t-1}\right\| 
\end{array}
\end{equation}

For the differences in the parameters, we now derive,
\begin{equation}\label{eq:avgapproxomega}
\begin{array}{l}
\mathbb{E}\left\|\tilde{\bm{\omega}}^{t+1}-\bar{\bm{\omega}}^{t+1}\right\| \le \\ 
\mathbb{E}\left\|\frac{1}{m}(\bm{I}\otimes\bm{1} \bm{1}^T)\left[
\mathcal{M}(\bm{\omega}^t+\epsilon(\bm{p}^t+\epsilon \bm{g}^{t+1}),\bm{\omega}^t,\aleph^{t+1},u^t) - \mathcal{M}(\tilde{\bm{\omega}}^t+\epsilon(\bm{p}^t+\epsilon \tilde{\bm{g}}^{t+1}),\tilde{\bm{\omega}}^t,\tilde\aleph^{t+1},u^t)\right]\right\| \\ \le 
\mathbb{E}\left\|\frac{1}{m}(\bm{I}\otimes\bm{1} \bm{1}^T)(\bm{\omega}^t-\tilde{\bm{\omega}}^t)\right\|
+ \epsilon^2\mathbb{E}\left\|\frac{1}{m}(\bm{I}\otimes\bm{1} \bm{1}^T)(\bm{g}^{t+1}-\tilde{\bm{g}}^{t+1})\right\|
\\ \quad +\mathbb{P}\left[(\mathcal{M}(\bm{\omega}^t+\epsilon(\bm{p}^t+\epsilon \bm{g}^{t+1}),\bm{\omega}^t,\aleph^{t+1},u^t) = \bm{\omega}^t+\epsilon(\bm{p}^t+\epsilon \bm{g}^{t+1})) \right. \\ \qquad\qquad\left.\cap (\mathcal{M}(\tilde{\bm{\omega}}^t+\epsilon(\bm{p}^t+\epsilon \tilde{\bm{g}}^{t+1}),\tilde{\bm{\omega}}^t,\tilde\aleph^{t+1},u^t) = \tilde{\bm{\omega}}^{t})\right] \\\qquad\qquad\qquad\times
\mathbb{E}\left\|\epsilon(\bm{p}^t+\epsilon (\bm{g}^{t+1}-\bar{\bm{g}}^{t+1}+\bar{\bm{g}}^{t+1}-\tilde{\bm{g}}^{t+1}+\tilde{\bm{g}}^{t+1}))\right\|
\\
\qquad +\mathbb{P}\left[(\mathcal{M}(\bm{\omega}^t+\epsilon(\bm{p}^t+\epsilon \bm{g}^{t+1}),\bm{\omega}^t,\aleph^{t+1},u^t) = \bm{\omega}^t) \right. \\ \qquad\qquad\left.\cap (\mathcal{M}(\tilde{\bm{\omega}}^t+\epsilon(\bm{p}^t+\epsilon \tilde{\bm{g}}^{t+1}),\tilde{\bm{\omega}}^t,\tilde\aleph^{t+1},u^t) = \tilde{\bm{\omega}}^{t}+\epsilon(\bm{p}^t+\epsilon \tilde{\bm{g}}^{t+1})\right] \\\qquad\qquad\qquad\times
\mathbb{E}\left\|\epsilon(\bm{p}^t+\epsilon \tilde{\bm{g}}^{t+1})\right\|
\\ \le
\mathbb{E}\left\|\bar{\bm{\omega}}^t-\tilde{\bm{\omega}}^t\right\|
+ \epsilon^2\mathbb{E}\left\|\bar{\bm{g}}^{t+1}-\tilde{\bm{g}}^{t+1}\right\|\\ \quad +\epsilon\mathbb{E}\left[\left\|\tilde\aleph^{t+1}-\bar\aleph^{t+1}+\bar{\aleph}^{t+1}-\aleph^{t+1}\right\|
\left[\|\bm{p}^t\|+\epsilon\|\tilde{\bm{g}}^{t+1}\|+\epsilon\|\tilde{\bm{g}}^{t+1}-\bar{\bm{g}}^{t+1}\|+\epsilon \|\bar{\bm{g}}^{t+1}-\bm{g}^{t+1}\|
\right]\right] \\
\le
\mathbb{E}\left\|\bar{\bm{\omega}}^t-\tilde{\bm{\omega}}^t\right\| \\ \qquad 
+ \epsilon^2\left[\mathbb{E}\left\|\tilde{\bm{g}}^{t}-\bar{\bm{g}}^t\right\|+
L_2 \mathbb{E}\left\|\tilde{\bm{\omega}}^{t}-\bar{\bm{\omega}}^t\right\|+L_2\mathbb{E} \left\|{\bm{\omega}}^{t}-\bar{\bm{\omega}}^t\right\| +
L_2 \mathbb{E}\left\|\tilde{\bm{\omega}}^{t-1}-\bar{\bm{\omega}}^{t-1}\right\|+L_2 \mathbb{E}\left\|{\bm{\omega}}^{t-1}-\bar{\bm{\omega}}^{t-1}\right\| 
\right]\\ \quad +\epsilon\mathbb{E}\left[\left[\left\|\tilde{{\aleph}}^{t}-\bar{{\aleph}}^t\right\|+
L_3 \left\|\tilde{\bm{\omega}}^{t}-\bar{\bm{\omega}}^t\right\|+L_3 \left\|{\bm{\omega}}^{t}-\bar{\bm{\omega}}^t\right\|\right.\right. \\ \qquad\qquad\qquad+\left.\left.
L_3\left\|\tilde{\bm{\omega}}^{t-1}-\bar{\bm{\omega}}^{t-1}\right\|+L_3  \left\|{\bm{\omega}}^{t-1}-\bar{\bm{\omega}}^{t-1}\right\| 
+\left\|\bar\aleph^{t+1}-\aleph^{t+1}\right\|
\right]\times\right. \\ \quad \left.
\left[\|\bm{p}^t\|+\epsilon\|\tilde{\bm{g}}^{t+1}\|+\epsilon\left\|\tilde{\bm{g}}^{t}-\bar{\bm{g}}^t\right\|+
\epsilon L_2\left\|\tilde{\bm{\omega}}^{t}-\bar{\bm{\omega}}^t\right\|+\epsilon L_2 \left\|{\bm{\omega}}^{t}-\bar{\bm{\omega}}^t\right\| +
\epsilon L_2 \left\|\tilde{\bm{\omega}}^{t-1}-\bar{\bm{\omega}}^{t-1}\right\|\right.\right. \\ \qquad\qquad \left.\left.+\epsilon L_2 \left\|{\bm{\omega}}^{t-1}-\bar{\bm{\omega}}^{t-1}\right\|+\epsilon \|\bar{\bm{g}}^{t+1}-\bm{g}^{t+1}\|
\right]\right]  \\ \le
\mathbb{E}\left\|\bar{\bm{\omega}}^t-\tilde{\bm{\omega}}^t\right\| \\ \qquad 
+ \epsilon^2\left[\mathbb{E}\left\|\tilde{\bm{g}}^{t}-\bar{\bm{g}}^t\right\|+
L_2 \mathbb{E}\left\|\tilde{\bm{\omega}}^{t}-\bar{\bm{\omega}}^t\right\|+L_2\mathbb{E} \left\|{\bm{\omega}}^{t}-\bar{\bm{\omega}}^t\right\| +
L_2 \mathbb{E}\left\|\tilde{\bm{\omega}}^{t-1}-\bar{\bm{\omega}}^{t-1}\right\|+L_2 \mathbb{E}\left\|{\bm{\omega}}^{t-1}-\bar{\bm{\omega}}^{t-1}\right\| 
\right] \\ \qquad + 8\epsilon\mathbb{E}\left\|\tilde{\aleph}^{t}-\bar{\aleph}^t\right\|^2+8 \epsilon L_3 \mathbb{E}\left\|\tilde{\bm{\omega}}^{t}-\bar{\bm{\omega}}^t\right\|^2+8\epsilon L_3 \mathbb{E}\left\|{\bm{\omega}}^{t}-\bar{\bm{\omega}}^t\right\|^2+8\epsilon L_3 \mathbb{E}\left\|\tilde{\bm{\omega}}^{t-1}-\bar{\bm{\omega}}^{t-1}\right\|^2+8\epsilon L_3 \mathbb{E}\left\|{\bm{\omega}}^{t-1}-\bar{\bm{\omega}}^{t-1}\right\|^2\\ \qquad
+8\epsilon L_2 \mathbb{E}\left\|{\aleph}^{t}-\bar{\aleph}^t\right\|^2+8\epsilon\mathbb{E}\left\|\bm{p}^t\right\|^2+8\epsilon^2 \mathbb{E}\left\|\tilde{\bm{g}}^{t+1}\right\|^2+8\epsilon^2 \mathbb{E}\left\|\tilde{\bm{g}}^t-\bar{\bm{g}}^t\right\|^2+8\epsilon^2 L_2  \mathbb{E}\left\|\tilde{\bm{\omega}}^t-\bar{\bm{\omega}}^t\right\|^2 \\ \qquad + 
8\epsilon^2 L_2  \mathbb{E}\left\|{\bm{\omega}}^t-\bar{\bm{\omega}}^t\right\|^2 + 
8\epsilon^2 L_2  \mathbb{E}\left\|\tilde{\bm{\omega}}^{t-1}-\bar{\bm{\omega}}^{t-1}\right\|^2  + 
8\epsilon^2 L_2  \mathbb{E}\left\|{\bm{\omega}}^{t-1}-\bar{\bm{\omega}}^{t-1}\right\|^2  + 
8\epsilon^2 L_2  \mathbb{E}\left\|{\bm{g}}^{t+1}-\bar{\bm{g}}^{t+1}\right\|^2  \\ 
\le \left(1+\epsilon^2 L_2\right)\mathbb{E}\left\|\bar{\bm{\omega}}^t-\tilde{\bm{\omega}}^t\right\|+8  \epsilon\left(L_3+\epsilon L_2\right)\mathbb{E}\left\|\bar{\bm{\omega}}^t-\tilde{\bm{\omega}}^t\right\|^2+\epsilon^2 
\mathbb{E}\left\|\tilde{\bm{g}}^{t}-\bar{\bm{g}}^t\right\|
+8\epsilon^2 
\mathbb{E}\left\|\tilde{\bm{g}}^{t}-\bar{\bm{g}}^t\right\|^2 \\ 
\quad+8\epsilon \mathbb{E}\left\|\tilde{\aleph}^t-\bar{\aleph}^t\right\|^2 +8\epsilon  (L_3+\epsilon L_2)\mathbb{E}\left\|\tilde{\bm{\omega}}^{t-1}-\bar{\bm{\omega}}^{t-1}\right\|^2 \\ 
\quad + \epsilon^2 L_2 \mathbb{E}\left\|\bm{\omega}^t-\bar{\bm{\omega}}^t\right\|
+ \epsilon^2 L_2 \mathbb{E}\left\|\bm{\omega}^{t-1}-\bar{\bm{\omega}}^{t-1}\right\|
+ 8\epsilon (L_3+\epsilon L_2) \mathbb{E}\left\|\bm{\omega}^t-\bar{\bm{\omega}}^t\right\|^2
+ 8\epsilon (L_3+\epsilon L_2) \mathbb{E}\left\|\bm{\omega}^{t-1}-\bar{\bm{\omega}}^{t-1}\right\|^2 \\ 
\quad + 8\epsilon L_2 \mathbb{E}\|\aleph^t-\bar{\aleph}^t\|^2
+ 8\epsilon^2 L_2 \mathbb{E}\left\|\bm{g}^{t+1}-\bar{\bm{g}}^{t+1}\right\|^2
+8\epsilon \mathbb{E}\|\bm{p}^t\|^2+8\epsilon^2\mathbb{E}\left\|\tilde{\bm{g}}^{t+1}\right\|^2
\end{array}
\end{equation}
Putting these together, we get,
\begin{equation}\label{eq:matrecursion}
\begin{array}{l}
\begin{pmatrix}
\mathbb{E}\left\|\tilde{\bm{\omega}}^{t+1}-\bar{\bm{\omega}}^{t+1}\right\| \\
\mathbb{E}\left\|\tilde{\bm{g}}^{t+1}-\bar{\bm{g}}^{t+1}\right\| \\
\mathbb{E}\left\|\tilde{{\aleph}}^{t+1}-\bar{{\aleph}}^{t+1}\right\|
\end{pmatrix}
\le 
\mathbf{A_1} \begin{pmatrix}
\mathbb{E}\left\|\tilde{\bm{\omega}}^{t}-\bar{\bm{\omega}}^{t}\right\| \\
\mathbb{E}\left\|\tilde{\bm{g}}^{t}-\bar{\bm{g}}^{t}\right\| \\
\mathbb{E}\left\|\tilde{{\aleph}}^{t}-\bar{{\aleph}}^{t}\right\|
\end{pmatrix}+\mathbf{A_2} \begin{pmatrix}
\mathbb{E}\left\|\tilde{\bm{\omega}}^{t}-\bar{\bm{\omega}}^{t}\right\|^2 \\
\mathbb{E}\left\|\tilde{\bm{g}}^{t}-\bar{\bm{g}}^{t}\right\|^2 \\
\mathbb{E}\left\|\tilde{{\aleph}}^{t}-\bar{{\aleph}}^{t}\right\|^2
\end{pmatrix} \\ 
\qquad + \mathbf{A_3} \begin{pmatrix}
\mathbb{E}\left\|\tilde{\bm{\omega}}^{t-1}-\bar{\bm{\omega}}^{t-1}\right\| \\
\mathbb{E}\left\|\tilde{\bm{g}}^{t-1}-\bar{\bm{g}}^{t-1}\right\| \\
\mathbb{E}\left\|\tilde{{\aleph}}^{t-1}-\bar{{\aleph}}^{t-1}\right\|
\end{pmatrix}+\mathbf{A_4} \begin{pmatrix}
\mathbb{E}\left\|\tilde{\bm{\omega}}^{t-1}-\bar{\bm{\omega}}^{t-1}\right\|^2 \\
\mathbb{E}\left\|\tilde{\bm{g}}^{t-1}-\bar{\bm{g}}^{t-1}\right\|^2 \\
\mathbb{E}\left\|\tilde{{\aleph}}^{t-1}-\bar{{\aleph}}^{t-1}\right\|^2
\end{pmatrix}
+\mathbf{B} \\ 
\mathbf{A}_1 = \begin{pmatrix} (1+\epsilon^2 L_2) & \epsilon^2 & 8\epsilon \\
L_2 & 1 & 0 \\ L_2 & 0 & 1 \end{pmatrix} \\
\mathbf{A}_2 = \begin{pmatrix}
8\epsilon (L_3+\epsilon L_2) & 8\epsilon^2 &  8\epsilon \\ 0 & 0 & 0 \\ 0 & 0 & 0 
\end{pmatrix} \\
\mathbf{A}_3 = \begin{pmatrix}
\epsilon^2 L_2 & 0 &  8\epsilon \\ L_2 & 0 & 0 \\ L_2 & 0 & 0 
\end{pmatrix} \\
\mathbf{A}_4 = \begin{pmatrix}
8\epsilon (L_3+\epsilon L_2) & 0 &  0 \\ 0 & 0 & 0 \\ 0 & 0 & 0 
\end{pmatrix} \\
\mathbf{B} = 
\begin{pmatrix}
\frac{2\epsilon^3  L_2 \|\mathbf{M}\| }{1-\beta}+\frac{16 \epsilon^3L_2 (1+\epsilon)\|\mathbf{M}\|^2 }{1-\beta^2} +8\epsilon \|\mathbf{M}\|^2+8\epsilon^2 U\\ 
\frac{4\epsilon^3  L_2 \|\mathbf{M}\| }{1-\beta} \\
\frac{4\epsilon^3  L_2 \|\mathbf{M}\| }{1-\beta}
\end{pmatrix}
\end{array}
\end{equation}
We can obtain the rough inductive upper bound stated in the Theorem by,
\[
\begin{array}{l}
\begin{pmatrix}
\mathbb{E}\left\|\tilde{\bm{\omega}}^{t}-\bar{\bm{\omega}}^{t}\right\| \\
\mathbb{E}\left\|\tilde{\bm{g}}^{t}-\bar{\bm{g}}^{t}\right\| \\
\mathbb{E}\left\|\tilde{{\aleph}}^{t}-\bar{{\aleph}}^{t}\right\|
\end{pmatrix}
\le \left\|\left(\mathbf{A}_1+\mathbf{A}_2+\mathbf{A}_3+\mathbf{A}_4\right)^{2t} \mathbf{B}+\left(\mathbf{A}_1+\mathbf{A}_2+\mathbf{A}_3+\mathbf{A}_4\right)^{2t} \mathbf{B}^{\cdot 2}\right\|  
\le \epsilon^3 \tilde{C}^{2t}  B,\\
\tilde{C} = 4L_2+1+\epsilon (8+16L_3)+18 \epsilon^2 L_2,\\
B = \|\mathbf{B}\|
\end{array}
\]
\end{proof}

\section{Euler Update}

\begin{algorithm}[h!]
\SetAlgoLined
\KwResult{$\bm{\omega}^*, \mathbf{p}^*$}
 input: $\bm{\omega}, \mathbf{p}, \epsilon, \mathbf{g}$\;
 $\mathbf{p}^*  = \mathbf{p} + \epsilon \mathbf{g}$\;
 $\bm{\omega}^*  = \bm{\omega} + \epsilon \mathbf{p}^*$\;
 \caption{``EulerUpdate'' (1st order Euler integrator)}
\end{algorithm}

\section{Experiment Hyperparameters}

\textbf{Linear Regression.} We set the doubly stochastic matrix, $\mathbf{W} = \frac{1}{N_a}\mathbf{1}_4$, where number of agents, $N_a = 4$ and $\mathbf{1}_4$ is a $4\times4$ matrix of ones. We run the experiment over 9 seeds for $T = 10^5$ iterations. Hardware: MacBook Pro, Processor: 2.6 GHz 6-Core Intel Core i7, Memory: 16 GB.
\begin{itemize}
    \item \textbf{Centralized HMC:} $\epsilon = 4\times10^{-4}$, $L = 1$, prior precision $= 1.0$.
    \item \textbf{Decentralized MALA:} $\epsilon = 4\times10^{-4}$, prior precision $= 1.0$. We switch off the MH step for the first $10^3$ steps to ensure that the Taylor approximation is only applied from a point closer to the target distribution.
    \item \textbf{Decentralized ULA:} $\epsilon = 3\times10^{-7}$. Following the same notation from \cite{parayil2020decentralized}: $\beta_0 = 0.48,\delta_1 = 0.01, \delta_2 = 0.55,  b_1 = 230, b_2 = 230$.
\end{itemize}

\textbf{Logistic Regression.}

\textbf{Partial Observation:}
 We set the doubly stochastic matrix, $\mathbf{W} = \frac{1}{N_a}\mathbf{1}_4$, where number of agents, $N_a = 4$ and $\mathbf{1}_4$ is a $4\times4$ matrix of ones. We run the experiment over 9 seeds for $T = 8\times10^3$ iterations. Hardware: GeForce RTX 2080 Ti. 
\begin{itemize}
    \item \textbf{Centralized HMC:} $\epsilon = 0.001$, $L = 1$, prior precision $= 100.0$.
    \item \textbf{Decentralized MALA:} $\epsilon = 5\times10^{-4}$, prior precision $= 100.0$. We switch off the MH step for the first $2\times 10^3$ steps to ensure that the Taylor approximation is only applied from a point closer to the target distribution.
    \item \textbf{Decentralized ULA:} $\epsilon = 1\times10^{-5}$. Following the same notation from \cite{parayil2020decentralized}: $\beta_0 = 0.48,\delta_1 = 0.01, \delta_2 = 0.55,  b_1 = 230, b_2 = 230$.
\end{itemize}

\textbf{Ring Network:}
 We set the doubly stochastic matrix, $\mathbf{W} = (\mathbf{I} + \mathbf{A})\frac{1}{N_a}$, where number of agents, $N_a = 5$ and $\mathbf{I}$ is the identity matrix, and $\mathbf{A}$ is the adjacency matrix for a ring shaped graph. We run the experiment over 9 seeds for $T = 1\times10^4$ iterations. Hardware: GeForce RTX 2080 Ti. 
\begin{itemize}
    \item \textbf{Centralized HMC:} $\epsilon = 0.001$, $L = 1$, prior precision $= 100.0$.
    \item \textbf{Decentralized MALA:} $\epsilon = 0.003$, prior precision $= 100.0$. We switch off the MH step for the first $1\times 10^3$ steps to ensure that the Taylor approximation is only applied from a point closer to the target distribution.
    \item \textbf{Decentralized ULA:} $\epsilon = 1\times10^{-4}$. Following the same notation from \cite{parayil2020decentralized}: $\beta_0 = 0.48,\delta_1 = 0.01, \delta_2 = 0.55,  b_1 = 230, b_2 = 230$.
\end{itemize}

\textbf{Bayesian neural network.}
We set the doubly stochastic matrix, $\mathbf{W} = \frac{1}{N_a}\mathbf{1}_2$, where number of agents, $N_a = 2$ and $\mathbf{1}_2$ is a $2\times2$ matrix of ones. We run the experiment for $T = 5\times10^5$ iterations.
\textbf{Decentralized MALA:} $\epsilon = 7\times10^{-5}$, prior precision $= 10.0$. We switch off the MH step for the first $2\times 10^3$ steps to ensure that the Taylor approximation is only applied from a point closer to the target distribution. Hardware: GeForce RTX 2080 Ti. 

\end{document}